\documentclass[sigconf]{acmart}

\pdfoutput=1 

\usepackage{graphicx}

\usepackage{subfigure}
\usepackage{amsmath}
\usepackage{amsfonts}
\usepackage{amssymb}
\usepackage{amsxtra}
\usepackage{amsthm}

\usepackage{float}

\usepackage{calc}

\usepackage{cancel}

\usepackage{mathrsfs}
\usepackage[linesnumbered,lined,boxed,commentsnumbered]{algorithm2e}

\newtheorem{theorem}{Theorem}
\newtheorem{definition}{Definition}

\newtheorem{corollary}{Corollary}
\newtheorem{proposition}{Proposition}

\newtheorem{problem}{Problem}

\usepackage{soul}

\usepackage{color} 

\renewcommand{\b}[1]{{{{\color{blue}#1}}}} 

\newcommand{\nn}{{\mathscr{N}\negthickspace\negthickspace\negthinspace\mathscr{N}}\negthinspace}
\newcommand{\overbar}[1]{\mkern 3.0mu\overline{\mkern-2.5mu#1\mkern-2.0mu}\mkern 2.0mu}

\renewcommand\footnotetextcopyrightpermission[1]{}

\acmConference[BLANK]{BLANK}{BLANK}{BLANK}

\title{AReN: Assured ReLU NN Architecture for Model Predictive Control of LTI Systems}
\author{James Ferlez}
\affiliation{
	\institution{University of California, Irvine}
	\department{Dept. of Electrical Engineering and Computer Science}
}
\email{jferlez@uci.edu}
\author{Yasser Shoukry}
\affiliation{
	\institution{University of California, Irvine}
	\department{Dept. of Electrical Engineering and Computer Science}
}
\email{yshoukry@uci.edu}

\begin{document}

\begin{abstract}
	In this paper, we consider the problem of automatically designing a 
	Rectified Linear Unit (ReLU) Neural Network (NN) architecture that is 
	sufficient to implement the optimal Model Predictive Control (MPC) strategy 
	for an LTI system with quadratic cost. Specifically, we propose AReN, an 
	algorithm to generate Assured ReLU Architectures. AReN takes as input an 
	LTI system with quadratic cost specification, and outputs a ReLU NN 
	architecture with the assurance that there exist network weights that 
	exactly implement the associated MPC controller. AReN thus offers new 
	insight into the design of ReLU NN architectures for the control of LTI 
	systems: instead of training a heuristically chosen NN architecture on data 
	-- or iterating over many architectures until a suitable one is found -- 
	AReN can suggest an adequate NN architecture before training begins. While 
	several previous works were inspired by the fact that both ReLU NN 
	controllers and optimal MPC controller are both Continuous, 
	Piecewise-Linear (CPWL) functions, exploiting this similarity to design NN 
	architectures with correctness guarantees has remained elusive. AReN 
	achieves this using two novel features. First, we reinterpret a recent 
	result about the implementation of CPWL functions via ReLU NNs to show that 
	a CPWL function may be implemented by a ReLU architecture that is 
	determined by the number of distinct affine regions in the function. 
	Second, we show that we can efficiently over-approximate the number of 
	affine regions in the optimal MPC controller without solving the MPC 
	problem exactly. Together, these results connect the MPC problem to a ReLU 
	NN implementation without explicitly solving the MPC and directly 
	translates this feature to a ReLU NN architecture that comes with the 
	assurance that it can implement the MPC controller. We show through 
	numerical results the effectiveness of AReN in designing an NN architecture.
\end{abstract}

\maketitle


\section{Introduction} 
\label{sec:introduction}
End-to-end learning is attractive for the realization of autonomous 
cyber-physical systems, thanks to the appeal of control systems based on a pure 
data-driven architecture. By taking advantage of the current advances in the 
field of reinforcement learning, several works in the literature showed how a 
well trained deep NN that is capable of controlling cyber-physical systems to 
achieve certain tasks~\cite{bojarski2016end}. Nevertheless, the current state-of-the-art 
practices of designing these deep NN-based controllers is based on heuristics 
and hand-picked hyper-parameters (e.g., number of layers, number of neurons per 
layer, training parameters, training algorithm) without an underlying theory 
that guides their design. In this paper, we focus on the fundamental question 
of how to systematically choose the NN architecture (number of layers and 
number of neurons per layer) such that we guarantee the correctness of the 
chosen NN architecture.

In this paper, we will confine our attention to the state-feedback Model 
Predictive Control (MPC) of a Linear Time-Invariant (LTI) system with quadratic 
cost and under input and output constraints (see Section 
\ref{sec:problem_formulation} for the specific MPC formulation). Importantly, 
this MPC control problem is known to have a solution that is Continuous and 
Piecewise-Linear (CPWL) \footnote{Although these functions are in fact 
continuous, piecewise-\emph{affine}, the literature on the subject refer to 
them as  piecewise ``linear'' functions, and hence we will conform to that 
standard.} in the current system state 
\cite{BemporadExplicitLinearQuadratic2002}. This property renders optimal MPC 
controllers compatible with a ReLU NN implementation, as any ReLU NN defines a 
CPWL function of its inputs. For this reason, several recent papers focus on 
how to approximate an optimal MPC controller using a ReLU 
NN~\cite{ChenApproximatingExplicitModel2018}.

However, unlike other work on the subject, AReN seeks to use knowledge of the 
underlying control problem to guide the design of \emph{data-trained NN 
controllers}. One of the outstanding problems with data-driven approaches is 
that the architecture for the NN is chosen either according to heuristics or 
else via a computationally expensive iteration scheme that involves adapting 
the architecture iteratively and re-training the NN. Besides being 
computationally taxing, neither of these provide any assurances that the 
resultant architecture is sufficient to adequately control the underlying 
system, either in terms of performance or stability. In the context of 
controlling an LTI system, then, AReN partially addresses these shortcomings: 
AReN is a computationally pragmatic algorithm that returns a ReLU NN 
architecture that is at least sufficient to implement the optimal MPC 
controller described before. That is given an LTI system with quadratic cost 
and input/output constraints, AReN determines a ReLU NN architecture -- both 
its structure and its size -- with the guarantee that \emph{there exists} an 
assignment of the weights such that the resultant NN \emph{exactly} implements 
the optimal MPC controller. This architecture can then be trained on data to 
obtain the final controller, only now with the assurance that the training 
algorithm \emph{can choose the optimal MPC controller} among all of the 
possible NN weight assignments available to it.

The algorithm we propose depends on two observations:
\begin{itemize}
	\item First, that any CPWL function may be translated into a ReLU NN 
		with an architecture determined wholly by the number of linear regions 
		in the function; this comes from a careful interpretation of the recent 
		results in \cite{AroraUnderstandingDeepNeural2016}, which are in turn 
		based on the hinging-hyperplane characterization of CPWL functions in 
		\cite{ShuningWangGeneralizationHingingHyperplanes2005} and the lattice 
		characterization of CPWL functions in 
		\cite{KahlertGeneralizedCanonicalPiecewiselinear1990}.

	\item Second, that there is a computationally efficient way to  
		over-approximate the number of linear regions in the optimal MPC 
		controller \emph{without solving for the optimal controller 
		explicitly}. This involves converting the state-region specification 
		equation for the optimal controller into a single, state-independent 
		collection of linear-inequality feasibility problems -- at the expense 
		of over-counting the number of affine regions that might be present in 
		the optimal MPC controller. This requires an algorithmic solution 
		rather than a closed form one, but the resultant algorithm is 
		computationally efficient enough to treat much larger problems than are 
		possible when the explicit optimal controller is sought.
\end{itemize}
Together these observations almost completely specify an algorithm that provides the architectural guidance we claim.

\noindent \textbf{Related work:} The idea of training neural networks to mimic 
the behavior of model predictive controllers can be traced back to the late 
1990s where neural networks trained to imitate MPC controllers were used to 
navigate autonomous robots in the presence of obstacles (see for 
example~\cite{ortega1996mobile}, and the references within) and to stabilize 
highly nonlinear systems~\cite{cavagnari1999neural}. With the recent advances 
in both the fields of NN and MPC, several recent works have explored the idea 
of imitating the behavior of MPC controllers ~\cite{hertneck2018learning, 
aakesson2006neural, pereira2018mpc, amos2018differentiable, 
claviere2019trajectory}.
The focus of all this work was to blindly mimic a base MPC controller without 
exploiting the internal structure of the MPC controller to design the NN 
structure systematically. The closest to our work are the results reported 
in~\cite{karg2018efficient,ChenApproximatingExplicitModel2018}. In this line of 
work, the authors were motivated by the fact that both explicit state MPC and 
ReLU NNs are CPWL functions, and they studied how to compare the performance of 
trained NN and explicit state MPC controllers. Different than the results 
reported in ~\cite{karg2018efficient,ChenApproximatingExplicitModel2018}, we 
focus, in this paper, on how to bridge the insights of explicit MPC to provide 
a systematic way to design a ReLU NN architecture with correctness guarantees.

Another related line of work is the problem of Automatic Machine Learning 
(AutoML) and in particular the problem of hyperparameter (number of layers, 
number of neurons per layer, and learning algorithm parameters) optimization 
and tuning in deep NN, in general, and in deep reinforcement learning, in 
particular (see for example~\cite{pedregosa2016hyperparameter, 
bergstra2012random, paul2019fast, baker2016designing, quanming2018taking} and 
the references within). In this line of work, an iterative and exhaustive 
search through a manually specified subset of the hyperparameter space is 
performed. Such a search procedure is typically followed by the evaluation of 
some performance metric that is used to select the best hyperparameters. Unlike 
the results reported in this line of work, AReN does not iterate over several 
designs to choose one. Instead, AReN directly generates an NN architecture that 
is guaranteed to control the underlying physical system adequately. 



\section{Problem Formulation} 
\label{sec:problem_formulation}

\subsection{Dynamical Model and Neural Network Controller} 
\label{sub:model_predictive_control}
We consider a discrete-time Linear, Time-Invariant (LTI) dynamical system 
of the form:
\begin{align}
	x(t+1) &= A x(t) + B u(t), \qquad y(t) = C x(t) \label{eq:dynamics}
\end{align}
where $x(t) \in \mathbb{R}^n$ is the state vector at time $t \in \mathbb{N}$, $u(t) \in 
\mathbb{R}^m$ is the control vector, and $y(t) \in \mathbb{R}^l$ is the output vector. The matrices $A$, $B$, and $C$ represents the system dynamics and the output map and have appropriate dimensions.
Furthermore, we consider controlling \eqref{eq:dynamics} with a state feedback 
neural network controller $\nn$:
\begin{equation}
	\nn: \mathbb{R}^{n} \rightarrow \mathbb{R}^m
\end{equation}
while fulfilling the constraints:
\begin{equation}
	y_{\text{min}} \le y(t) \le y_{\text{max}}, \qquad u_{\text{min}} \le u(t) \le u_{\text{max}}
\end{equation}
at all time instances $t \ge 0$ where $y_{\text{min}}, y_{\text{max}}, u_{\text{min}}$ and $u_{\text{max}}$ are constant vectors of appropriate dimension with $y_{\text{min}} < y_{\text{max}}$ and $u_{\text{min}} < u_{\text{max}}$ (where $<$ is taken element-wise).

In particular, we consider a ($K$-layer) Rectified Linear Unit Neural Network (ReLU NN)
that is specified by composing $K$ \emph{layer} functions (or just 
\emph{layers}). A layer with $\mathfrak{i}$ inputs and $\mathfrak{o}$ outputs 
is specified by a $(\mathfrak{o} \times \mathfrak{i} )$ real-valued matrix of 
\emph{weights}, $W$, and a $(\mathfrak{o} \times 1)$ real-valued matrix of 
\emph{biases}, $b$, as follows:
\begin{align}
	L_{\theta} : \mathbb{R}^{\mathfrak{i}} &\rightarrow \mathbb{R}^{\mathfrak{o}} \notag\\
	      z &\mapsto \max\{ W z + b, 0 \}
\end{align}
where the $\max$ function is taken element-wise, and $\theta \triangleq (W,b)$ 
for brevity. Thus, a $K$-layer ReLU NN function as above is specified by $K$ 
layer functions $\{L_{\theta^{(i)}} : i = 1, \dots, K\}$ whose input and output 
dimensions are \emph{composable}: that is they satisfy $\mathfrak{i}_{i} = 
\mathfrak{o}_{i-1}: i = 2, \dots, K$. Specifically:
\begin{equation}
	\nn(x) = (L_{\theta^{(K)}} \circ L_{\theta^{(K-1)}} \circ \dots \circ L_{\theta^{(1)}})(x).
\end{equation}
When we wish to make the dependence on parameters explicit, we will index a 
ReLU function $\nn$ by a \emph{list of matrices} $\Theta \triangleq ( 
\theta^{(1)}, \dots , \theta^{(K)} )$ \footnote{That is $\Theta$ is not the 
concatenation of the $\theta^{(i)}$ into a single large matrix, so it preserves 
information about the sizes of the constituent $\theta^{(i)}$.}. Also, it is 
common to allow the final layer function to omit the $\max$ function 
altogether, and we will be explicit about this when it is the case.

Note that specifying the number of layers and the \emph{dimensions} of the 
associated matrices $\theta^{(i)} = (\; W^{(i)}, b^{(i)}\; )$ specifies the 
\emph{architecture} of the ReLU NN. Therefore, we will use:
\begin{equation}
	\text{Arch}(\Theta) \triangleq ( (n,\mathfrak{o}_{1}), (\mathfrak{i}_{2},\mathfrak{o}_{2}), \ldots, (\mathfrak{i}_{K-1},\mathfrak{o}_{K-1}), (\mathfrak{i}_{K}, m))
\end{equation}
to denote the architecture of the ReLU NN $\nn_{\Theta}$. Note that our 
definition is general enough since it allows the layers to be of different 
sizes, as long as $\mathfrak{o}_{i-1} = \mathfrak{i}_{i}$ for $i = 2, \dots, K$.

\subsection{Neural Network Architecture Specification}
We are interested in finding an architecture $\text{Arch}(\Theta)$ for the 
$\nn_{\Theta}$ such that it is guaranteed to have enough parameters to exactly 
mimic the input-output behavior of some base controller 
$\mu:\mathbb{R}^n\rightarrow\mathbb{R}^m$. Due to the popularity of using model 
predictive control schemes as a base 
controller~\cite{ortega1996mobile,cavagnari1999neural,hertneck2018learning,aakesson2006neural,pereira2018mpc,amos2018differentiable, 
claviere2019trajectory}, we consider finite-horizon roll-out Model Predictive 
Control (MPC) scheme as the base controller that the ReLU NN is trying to mimic 
its behavior.

Finite-horizon roll-out MPC maps the current state, $x(t)$, to the first control input obtained from the solution to an optimal control problem over a finite time horizon $N_y$ with the first $N_u$ control actions chosen open-loop and the remaining $N_y - N_u$ control actions determined by an a-priori-specified constant-gain state feedback. Since this control scheme involves solving an optimal control problem at each time $t$ (with initial state $x(t)$), we will use the notation $x_{t^\prime|t}$ to denote the ``predicted state'' at time $t^\prime > t$ from the initial state $x(t)$ supplied to the MPC controller (the same notation as in \cite{BemporadExplicitLinearQuadratic2002}).
In particular, for fixed matrices $P,Q \geq 0$, $R > 0$ and $K$, we define the 
cost function: 
\begin{multline}
	\label{eq:cost_fn}
	J(U,x(t)) \triangleq x^{\text{T}}_{t+N_y|t}\; P \; x_{t+N_y|t} \\
		+ \sum_{k = 0}^{N_y-1} \left[
			x^{\text{T}}_{t+k|t} \; Q \; x_{t+k|t} +
			u^{\text{T}}_{t+k} \; R \; u_{t+k|t}
		\right]
\end{multline}
as a function of an $(m \times N_c+1)$ control variables matrix: 
\begin{equation}
	U \triangleq \left[ {u_{t}}^\text{T}; {u_{t+1}}^{\text{T}}; \; \dots; \; 
{u_{t+N_c}}^\text{T} \right]^{\text{T}}.
\end{equation}
Then the MPC control law is specified as:
\begin{equation}
\label{eq:mpc_controller}
	\mu_\text{MPC} : x(t) \mapsto u^*_t
\end{equation}
where
\begin{equation}
	\left[ {u^*_{t}}^\text{T}; {u^*_{t+1}}^{\text{T}}; \; \dots; \; 
{u^*_{t+N_c}}^\text{T} \right]^{\text{T}} = \arg \min_{U} J(U,x(t))
\end{equation}
subject to the constraints:
\begin{align}
	&y_{\text{min}} \leq y_{t+k|t} \leq y_{\text{max}}	&k&= 1, \dots, N_c 		&~
														\label{eq:first_constraint} \\
	&u_{\text{min}} \leq u_{t+k} \leq u_{\text{max}}	&k&=0, 1, \dots, N_c 	&~ \\
	&x_{t|t} = x(t)										&~& 					&~ \\
	&x_{t+k+1|t} = A x_{t+k|t} + B u_{t+k}				&k&\geq 0 				&~
													  \label{eq:dynamics_constraint}\\
	&y_{t+k|t} = C x_{t+k|t}							&~& 					&~ \\
	&u_{t+k} = K x_{t+k|t}								&N&_u \leq k < N_y. 		&~  \label{eq:last_constraint}
\end{align}

The matrix $P$ is typically chosen to reflect the quadratic cost-to-go (using 
matrices $Q$ and $R$) resulting from the feedback control $K$ applied from 
time-step $t + N_y$ onwards (i.e. $P$ is the solution to the appropriate 
algebraic Ricatti equation). We will henceforth consider only this scenario, 
since this is the most common one; furthermore, since it doesn't benefit from 
$N_y >> N_c$, we will henceforth assume that $N_y-1 = N_c$. For future 
reference, this problem then has:
\begin{align}
	\omega &\triangleq m \cdot (N_c + 1) &\text{\emph{decision variables}; and} \label{eq:omega}\\
	\rho &\triangleq 2 \cdot l \cdot N_c + 2 \cdot m \cdot (N_c+1) &\text{\emph{inequality constraints}. \label{eq:rho}}
\end{align}

\subsection{Main Problem} 
\label{sub:main_problem}

We are now in a position to state the problem that we will consider in this 
paper.
\begin{problem}
\label{prob:main_problem}
	Given system matrices $A$, $B$ and $C$ (as in \eqref{eq:dynamics}); 
	performance matrices (cost function matrices) $P, Q \geq 0$, $R > 0$ (as in 
	\eqref{eq:cost_fn}); constant-gain feedback matrix $K$ (as in 
	\eqref{eq:last_constraint}); and integer horizon $N_c > 1$, choose a ReLU 
	NN \emph{architecture} $\text{Arch}(\Theta)$, such that there exists a 
	real-value assignment for the elements in $\Theta$ that renders:
	\begin{equation} 
			\nn_{\Theta} (x) = \mu_\text{MPC}(x)
	\end{equation}
	for all $x$ in some compact subset of $\mathbb{R}^n$.
\end{problem}



\section{Framework} 
\label{sec:framework}
As we have noted before, it is known that $\mu_\text{MPC}$ is a CPWL function 
\cite{BemporadExplicitLinearQuadratic2002}. However, CPWL functions are usually 
specified using both linear functions \emph{and} regions as in this example:
\begin{align}
f(x) = 
	\begin{cases}
		2x + 3& x > 0 \\
		-2x + 3& x \le 0 	
	\end{cases}.
	\label{eq:CPWL_example}
\end{align}
This specification is difficult to implement using ReLU NNs, though, because 
the structure of a ReLU neural networks \emph{intertwines} the implementation 
of the linear functions and their active regions.

Fortunately, there are several representations of CPWL functions that avoid the 
explicit specification of regions by ``encoding'' them into the composition of 
nonlinear functions with linear ones. Recent work \cite[Theorem 
2.1]{AroraUnderstandingDeepNeural2016} considered one such representation based 
on hinging hyperplanes \cite{ShuningWangGeneralizationHingingHyperplanes2005}, 
and showed that this representation can be translated easily into a ReLU neural 
network implementation, \emph{whenever the CPWL function is known explicitly}. 

Given the computational cost of computing $\mu_\text{MPC}$ explicitly, the 
chief difficulty in Problem \ref{prob:main_problem} thus lies in inferring the 
neural network architecture $\text{Arch}(\Theta)$ without access to the 
explicit MPC controller $\mu_\text{MPC}$. Unfortunately, the hinging hyperplane 
representation employed in \cite[Theorem 2.1]{AroraUnderstandingDeepNeural2016} 
cannot be easily used in this circumstance (for more about why this particular 
implementation is unsuitable when $\mu_\text{MPC}$ is not explicitly known, see 
also Section \ref{sec:discussion}.)

However, every CPWL function also has a \emph{(two-level) lattice 
representation} \cite{TarelaRegionConfigurationsRealizability1999} 
\footnote{The lattice representation is in fact an intermediary representation 
used to construct the hinging hyperplane representation; see 
\cite{ShuningWangGeneralizationHingingHyperplanes2005}.}: unlike the particular 
hinging hyperplane representation mentioned above, we will show that the 
lattice representation \emph{can} be used to solve Problem 
\ref{prob:main_problem} without explicitly solving for $\mu_\text{MPC}$. In 
particular, the lattice representation of a CPWL function has two properties 
that facilitate this:
\begin{enumerate}
	\item It has a structure that is amenable to implementation with a ReLU 
		NN (by mechanisms similar to those used in 
		\cite{AroraUnderstandingDeepNeural2016}); and

	\item It is described purely in terms of the \emph{local linear 
		functions} and the number of \emph{unique-order regions} (exact 
		definitions of these terms are given in the next subsection) in the 
		CPWL function, both of which we can efficiently over-approximate for 
		$\mu_\text{MPC}$.
\end{enumerate}
Thus, a description of the lattice representation largely explains how to solve 
Problem \ref{prob:main_problem}; we follow this discussion by connecting it to 
a top-level description of our algorithm.

\subsection{The Two-Level Lattice Representation of a CPWL Function} 
\label{sub:the_lattice_representation_of_a_cpwl_function}
To understand the lattice representation of a CPWL, we first need the following 
definition. Throughout this subsection we will assume that $f : \mathbb{R}^n 
\rightarrow \mathbb{R}$ is a CPWL function. All the subsequent discussion can 
be generalized directly to the case when $f : \mathbb{R}^n \rightarrow 
\mathbb{R}^m$.
\begin{definition}[Local Linear Function]
	\label{def:local_linear_function}
	Let $f : \mathbb{R}^n \rightarrow \mathbb{R}$ be a CPWL function. Then 
	$\ell$ is a \textbf{local linear function} of $f$ if there exists an open 
	set $U \subset D \subset \mathbb{R}^n$ such that for all $x \in U$:
	\begin{equation}
		f(x) = \ell(x).
	\end{equation}
	The set of all local linear functions will be denoted $\mathcal{R} = 
	\{\ell_1, \ell_2, \dots, \ell_N\}$.
\end{definition}
\noindent The CPWL function in \eqref{eq:CPWL_example} consists of the two 
local linear functions $\ell_1(x) = 2x+3$ and $\ell_2(x) = -2x+3$, for example.
 
The lattice representation is based on the following idea: Consider the set of 
\emph{distinct} local linear functions of $f$ namely $\{(1,\ell_1(x)), $ 
$\dots, (N,\ell_N(x))\}$ along with the natural projections of this set $\pi_1 
: (i,\ell_i(x)) \in \mathbb{N} \times \mathbb{R} \mapsto i$ and $\pi_2 : 
(i,\ell_i(x)) \in \mathbb{N} \times \mathbb{R} \mapsto \ell_i(x)$. It follows 
from the fact that the $f$ is \emph{continuous} PWL function that at least two 
local linear functions intersect for each $x$ on the \emph{boundary} between 
linear regions. Therefore, the ordering of the set $\{(1,\ell_1(x)), \dots, 
(N,\ell_N(x))\}$ by $\ge$ on the projection $\pi_2$ induces at least two 
different orderings of the projection $\pi_1$ (see 
Figure~\ref{fig:order_region_illustration} for an example). It is a profound 
observation nevertheless, because it means that the relative ordering of the 
values $\{\ell_1(x), \dots, \ell_N(x)\}$ can be used to decide which of the 
local linear function is ``active'' at a particular $x$. This is illustrated in 
Figure \ref{fig:order_region_illustration}; see also a similar figure in 
\cite[Figure 1]{TarelaRegionConfigurationsRealizability1999}.
\begin{figure}
	\centering
	\includegraphics[width=0.4\textwidth]{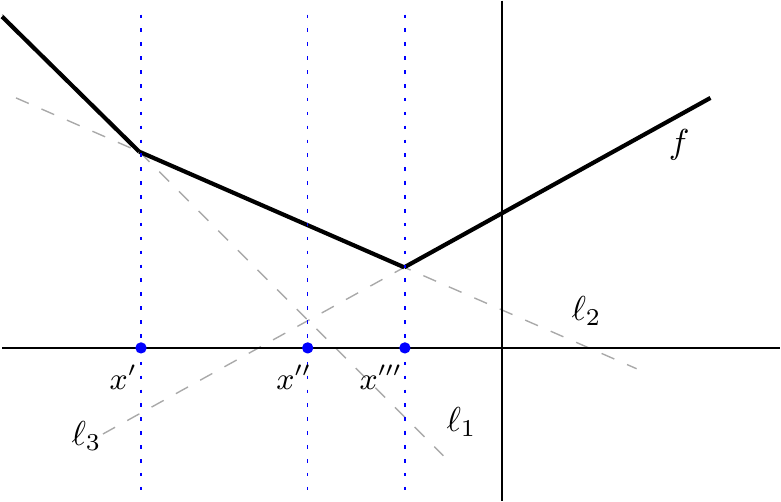}
	\caption{Ordering of local linear functions changes at the boundary between linear regions: $f$ is a CPWL function with local linear functions $\ell_1$, $\ell_2$ and $\ell_3$. Note that $\ell_1(x^\prime) \geq \ell_2(x^\prime) \geq \ell_3(x^\prime)$ and $\ell_2(x^\prime) \geq \ell_1(x^\prime) \geq \ell_3(x^\prime)$ are two different orderings at the boundary point $x'$. Also note that the ordering can change \emph{within} a linear region: c.f. $x^{\prime\prime}$. See also \cite[Figure 1]{TarelaRegionConfigurationsRealizability1999}.}
	\label{fig:order_region_illustration}
\end{figure}
This also suggests that we make the following definition, which allows us to 
talk about regions in the domain of $f$ over which the order of the local 
linear functions is the same.
\begin{definition}[Unique-Order Region (rephrasing of {\cite[Definition 2.3]{TarelaRegionConfigurationsRealizability1999}})]
	\label{def:order_region}
	Let $f : D \subset \mathbb{R}^n \rightarrow \mathbb{R}$ be a CPWL function 
	with $N$ distinct local linear functions $\mathcal{R} = \{\ell_1, \dots 
	\ell_N\}$; that is for all $x \in D$, $f(x) = \ell_i(x)$ for some $\ell_i 
	\in \mathcal{R}$. Then a \textbf{unique-order region} of $f$ is a region $O 
	\subseteq D$ from the hyperplane arrangement in $\mathbb{R}^{n}$ defined by 
	those hyperplanes $H_{ij} = \{ x : \ell_i(x) = \ell_j(x) \}$ that are 
	non-empty. In particular, for all $x$ in a unique-order region $O$, 
	$\ell_{i_1}(x) \geq \ell_{i_2}(x) \geq \dots \geq \ell_{i_N}(x)$ for some 
	permutation of $i_k$ of $\{1, \dots, N\}$.
\end{definition}

We are now in a position to describe the two-level lattice representation of a 
CPWL function. 
\begin{theorem}[Two-Level Lattice Forms From Unique-Order Regions {\cite[Theorem 4.1]{TarelaRegionConfigurationsRealizability1999}}]
\label{thm:lattice_form}
	Let $f$ be as in Definition \ref{def:order_region} with $M$ the number of 
	unique-order regions of $f$ in $D$. Then there exists at most $M$ subsets 
	$s_i \subseteq \{ 1, \dots, N \}$, $i = 1, \dots, M$ such that:
	\begin{equation}
		\label{eq:lattice_form}
		f(x) = \max_{1 \leq i \leq M} \min_{j \in s_i} \ell_{j}(x) \quad \forall x \in D.
	\end{equation}
\end{theorem}

\subsection{Structure of the Main Algorithm} 
\label{sub:main_algorithm}
Having described in detail the lattice representation of a CPWL, we return to 
the specific claims we made about how it structures our solution to Problem 
\ref{prob:main_problem}.

We first note that the form \eqref{eq:lattice_form} is well suited to 
implementation with a ReLU neural network: it is comprised of linear functions and 
$\max$/$\min$ functions, so many observations from 
\cite{AroraUnderstandingDeepNeural2016} apply to \eqref{eq:lattice_form} as 
well. In particular, the two-argument maximum function can be implemented 
directly with a ReLU using the well-known identity 
\begin{equation}
	\label{eq:max_identity}
	\max \{a,b\} = \frac{a+b}{2} + \frac{|a-b|}{2}
\end{equation}
and the following ReLU implementations of its constituent expressions 
\cite{AroraUnderstandingDeepNeural2016}:
\begin{align}
	|x| &= \max\{x,0\} + \max \{-x,0\} 
		\label{eq:abs_as_relu} \\
	x &= \max\{x,0\} - \max\{-x,0\}
		\label{eq:id_as_relu}.
\end{align}
Thus $\max$ can be implemented by a NN $\nn_{\Theta_{\max}}$ where:
\begin{equation}
	\Theta_{\max} = (
	\left[
		\begin{smallmatrix}
			1 & 1 \\
			-1 & -1 \\
			-1 & 1 \\
			1 & -1
		\end{smallmatrix}
	\right],
	\left[
		\begin{smallmatrix}
			\frac{1}{2} & -\frac{1}{2} & \frac{1}{2} & \frac{1}{2}
		\end{smallmatrix}
	\right]
	)
\end{equation}
This implementation is illustrated in Figure \ref{fig:relu_max_network}. Using 
the $\min$ variant of the identity \eqref{eq:max_identity}, namely:
\begin{equation}
	\min \{a,b\} = \frac{a+b}{2} - \frac{|a-b|}{2}
\end{equation}
leads to a similar ReLU implementation of the two-argument minimum function. In 
the previous notation, the \emph{architectures} of these $\max$ and $\min$ 
networks are the same, i.e. $\text{Arch}(\Theta_{\max})$ = 
$\text{Arch}(\Theta_{\min})$ = $( (2,4), (4,1) )$ with no activation function 
on the last layer.
\begin{figure}
	\centering
	\includegraphics[width=0.45\textwidth]{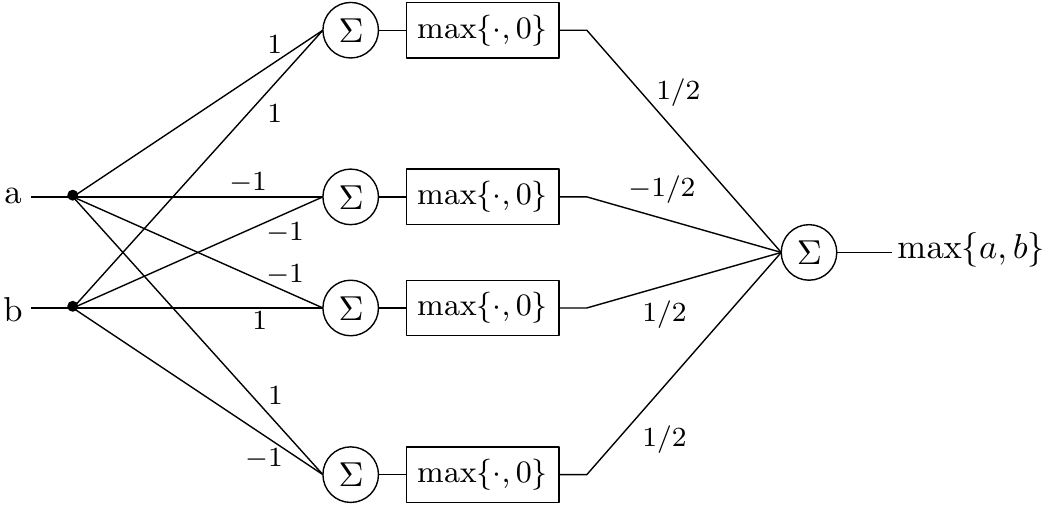}
	\caption{Illustration of a ReLU network to compute the maximum of two real numbers $a$ and $b$. See also \cite{AroraUnderstandingDeepNeural2016}.}
	\label{fig:relu_max_network} \vspace{-4mm}
\end{figure}

This implementation further suggests a natural way to implement the 
multi-element $\max$ (resp. $\min$) operation with a ReLU network 
\cite{AroraUnderstandingDeepNeural2016}. Such an operation can be implemented 
by deploying the two-element $\max$ (resp. $\min$) networks in a 
``divide-and-conquer'' fashion: the elements of the set to be maximized (resp. 
minimized) are fed pairwise into a \emph{layer} of two-element $\max$ (resp. 
$\min$) networks; the output of that first $\max$ (resp. $\min$) layer is fed 
pairwise into a \emph{subsequent} layer of two-element $\max$ (resp. $\min$) 
networks, and so on and so forth until there is only one output. Note that this 
approach can also be used on sets whose cardinality is not a power of two while 
maintaining a ReLU structure of the neural network $\nn$: the same value can  
be directed to multiple inputs as necessary. This structure is illustrated in 
Figure \ref{fig:relu_multi_max_network} for a network that computes the maximum 
of five real-valued inputs. Following this example, an $N$-input $\max$ (or 
$\min$) network $\max_N$ (resp. $\min_N$) is represented by a parameter list 
$\Theta_{{\max}_N}$ (resp. $\Theta_{{\min}_N}$) which has architecture:
\begin{figure*}
	\centering 
	\includegraphics[width=0.8\textwidth]{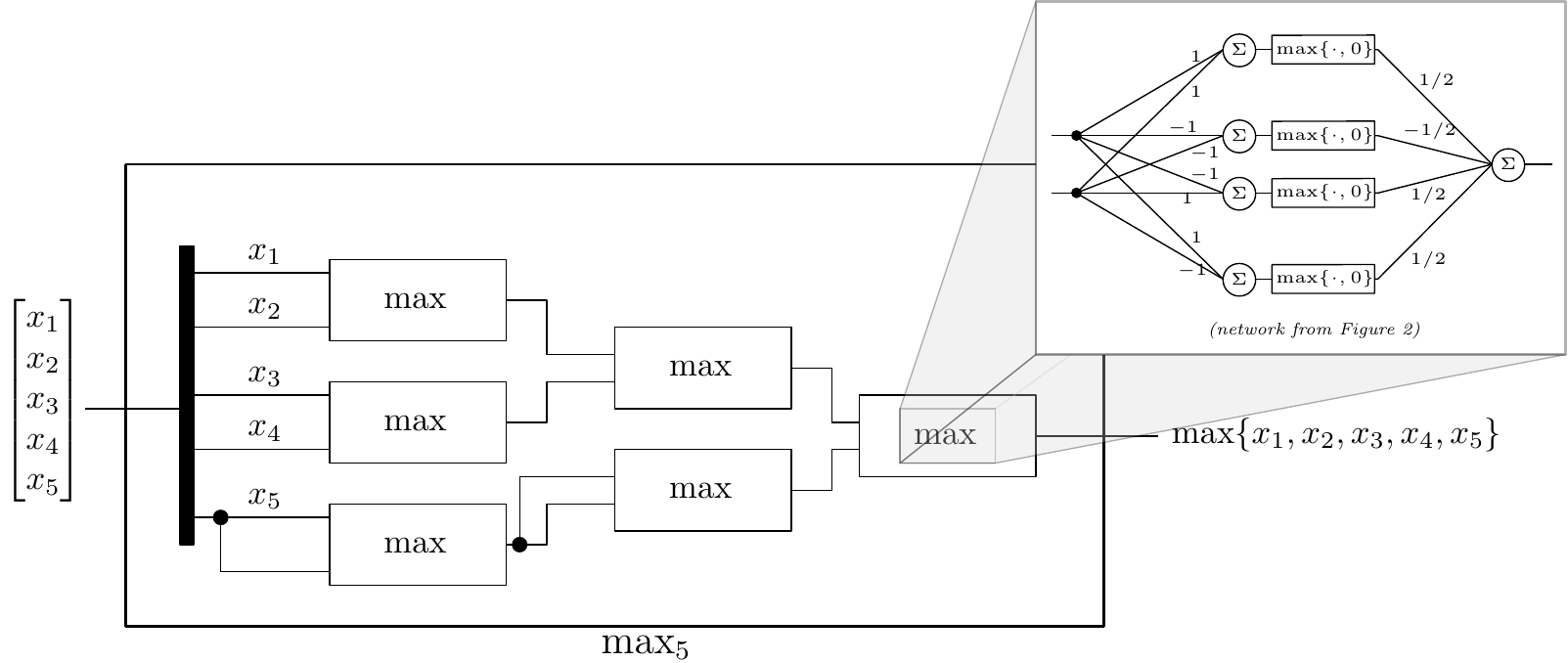} 
	\caption{Illustration of a ReLU network to compute the maximum of five real 
	numbers $\{x_1,x_2,x_3,x_4,x_5\}$. Callout depicts the network from Figure 
	\ref{fig:relu_max_network}. See also 
	\cite{AroraUnderstandingDeepNeural2016}.}\vspace{-3mm}
	\label{fig:relu_multi_max_network}
\end{figure*}
\begin{multline}
	\text{Arch}(\Theta_{{\max}_N}) = \text{Arch}(\Theta_{{\min}_N}) = \\
	\Big( \;
		\big(N, 2 \negthinspace\cdot\negthinspace \lceil  N/2 \rceil \big)\;, \;
		\lceil  N/2 \rceil \negthinspace\cdot\negthinspace \textnormal{Arch}( \Theta_{\max} ) \;, \; \hspace{45pt} \\
		(\lceil  N/2 \rceil - 1 ) \negthinspace\cdot\negthinspace \textnormal{Arch}( \Theta_{\max} ) \;, \;
		\dots \; , \;
		\textnormal{Arch}(\Theta_{\max}) \;
	\Big)
\end{multline}
where $c \cdot \text{Arch}(\Theta_{\max})$ means multiply every element in 
$\textnormal{Arch}(\Theta_{\max})$ by $c$; nested lists are ``flattened'' as 
appropriate; and there is no activation function on the final layer.

Now, given these multi-element $\max$/$\min$ networks, the remaining structure 
for a ReLU network implementation of the lattice form \eqref{eq:lattice_form} is clear: 
we need a neural network architecture capable of (i) implementing $f$'s local 
linear functions $\mathcal{R} = \{\ell_1, \dots, \ell_N\}$ and (ii) handling the 
selection of the subsets $s_i$. The implementation of the local linear 
functions is straightforward using a fully connected hidden layer. The 
selection can be handled by routing -- and replicating, as needed -- the output 
of those linear functions to a $\min$ network.  Since we do not know the exact 
size of the subsets $s_i$ in~\eqref{eq:lattice_form}, and hence the number of 
input ports for each $\min$ network, we must use $\min$ networks with as many 
pair-wise $\min$ input ports as there are local linear functions. Then, for 
subsets $s_i$ of size less than $N$, the architecture replicates some local 
functions multiple times to different input ports of the same $\min$ network to 
achieve the correct output. As discussed before, such replication does not 
affect the correctness of the architecture. Moreover, there will be one such 
$\min$ network for each unique-order region for a total of $M$.  This 
replicating and routing of signals can be accomplished by an auxiliary fully 
connected linear layer with $N$ inputs and $M\cdot N$ outputs. Since the 
purpose of this layer is to allow the weights to only select subsets of the 
local linear functions, this layer should have the property that all of the 
weights are either zero or one, and each output of the layer should select 
\emph{exactly} one input. That is the weight matrix for this layer should have 
a \emph{exactly} one $1$ in each row with all of the other weights set to $0$. 
For a real-valued CPWL function $f: \mathbb{R}^n \rightarrow \mathbb{R}$, this 
overall architecture is depicted in Figure \ref{fig:overall_architecture}. The 
selection and routing layer is depicted in red. The notation 
$\mathcal{R}_{\hat{s}_i}(x)$ reflects the routing and (one possible) 
replication of values of the local linear functions, and it is defined as 
follows:
\begin{multline}
\label{eq:redundant_replication}
	\mathcal{R}_{\hat{s}_i}(x) \triangleq \Big\{ \big( \{j\}, \ell_j(x) \big) : j \in s_i  \Big\} \\
	\cup \Big\{ \big( \{N+k\}, \ell_{\max s_i}(x) \big) : k = 1, \dots, N-|s_i| \Big\}.
\end{multline}
(That is $\mathcal{R}_{\hat{s}_i}(x)$ contains one copy of each of $\ell_j(x) : 
j \in s_i$, and as many additional copies of $\ell_{\max s_i}(x)$ as necessary 
to have $N$ total elements. In particular, $\max_{p \in 
\mathcal{R}_{\hat{s}_i}(x)} \pi_2 (p) = \max_{j \in s_i} \ell_j(x)$.)

Finally, we note that it is straightforward to extend this architecture to 
vector-valued functions like $\mu_\text{MPC}$. The structure of $\nn$ (Section 
\ref{sub:model_predictive_control}) means that a \emph{scalar} pairwise $\min$ 
(or $\max$) network can trivially compute the \emph{element-wise} minimum 
between two input vectors by simply allowing more inputs and applying the 
weights from Figure \ref{fig:relu_max_network} in an element-wise (diagonal) 
fashion. The result is an architecture that looks exactly like the one in 
Figure \ref{fig:overall_architecture}, only with the number of outputs $m$ 
multiplying the size of most signals.

The structure of the above described ReLU implementation is general enough to 
implement any CPWL function $f$ with $N$ local linear functions and $M$ unique 
order regions. We state this as a theorem below.
\begin{theorem}
\label{thm:nn_representation}
	Let $f : \mathbb{R}^n \rightarrow \mathbb{R}$ by CPWL with distinct local 
	linear functions $\mathcal{R} = \{\ell_1, \dots, \ell_N\}$ and $M$ unique 
	order regions. Then there is a parameter list $\Theta_{N,M}$ with: 
	\begin{multline}
	\label{eq:nn_representation}
		\textnormal{Arch}(\Theta_{N,M}) = \\
		\Big( \;
			\underbrace{(n, N\cdot M)}_{\text{linear layer}}, \;
			\underbrace{M\cdot\textnormal{Arch}(\Theta_{\min_N})}_{\min\text{ layer}} \;, \;
			\underbrace{\textnormal{Arch}(\Theta_{\max_M})}_{\max\text{ layer}}
		\; \Big)
	\end{multline}
	such that there exist an assignments for $\Theta_{N,M}$ that renders:
	$$\nn_{\Theta_{N,M}}(x) = f(x) \qquad \forall x \in \mathbb{R}^n,$$
	where $M\cdot\textnormal{Arch}(\Theta_{\min_N})$ means multiply every 
	element in $\textnormal{Arch}(\Theta_{\min_N})$ by $M$, and where nested 
	lists are ``flattened'' as needed. The final layers of the $\min$ layer and 
	the $\max$ layer lack activation functions.
\end{theorem}
\begin{proof}
	The proof is constructive: the discussion above explains the construction, 
	which is based on \cite[Theorem 
	4.1]{TarelaRegionConfigurationsRealizability1999}.
\end{proof}
\begin{corollary}
\label{cor:controller_arch}
	Any CPWL controller $\mu$ (such as $\mu_\text{MPC}$) can be implemented by 
	a ReLU network $\Theta$ with architecture $\textnormal{Arch}(\Theta_{N,M})$ 
	as described in Theorem \ref{thm:nn_representation}, where $N$ is the 
	number local linear functions of $\mu$ and $M$ is the number of 
	unique-order regions of $\mu$.
\end{corollary}

Note that in many cases it is hard to exactly know the parameters $N$ and $M$ exactly. The next result show that our correctness claims in Theorem~\ref{alg:estimateRegionCount} can be extended when an upper bound $\overbar{N} \ge N$ and $\overbar{M} \ge M$ is used to design the neural network architecture as explained in the next result.
\begin{theorem}
\label{thm:embedding}
	Let $\Theta_{N,M}$ be a parameter list such that 
	$\textnormal{Arch}(\Theta_{N,M})$ is as specified in Theorem 
	\ref{thm:nn_representation}, \eqref{eq:nn_representation}, and let 
	$\overbar{N} \geq N$ and $\overbar{M} \geq M$. Then there exists a 
	parameter list $\Theta_{\overbar{N},\overbar{M}}$ with 
	$\textnormal{Arch}(\Theta_{\overbar{N},\overbar{M}})$ as in 
	\eqref{eq:nn_representation} such that:
	\begin{equation}
		\nn_{\Theta_{N,M}}(x) = \nn_{\Theta_{\overbar{N},\overbar{M}}}(x) \quad \forall x.
	\end{equation}
\end{theorem}
\begin{proof}
	In order to implement the same function with a larger network, the extra 
	linear-layer neurons can simply duplicate calculations carried out by 
	neurons in the smaller network. For example, the extra neurons in the the 
	first linear layer can duplicate the calculation of $\ell_N$, and the extra 
	neurons in the second linear layer can duplicate the calculation of the 
	$M^\text{th}$ subset of $\{\ell_1, \dots, \ell_N\}$. This will not change 
	the output of the $\min$ and $\max$ layers.
\end{proof}
\noindent \emph{Note}: when ``embedding'' a smaller network, $\Theta_{N,M}$, 
into a larger one, $\Theta_{\overbar{N},\overbar{M}}$, it is incorrect to set 
the extra parameters in $\Theta_{\overbar{N},\overbar{M}}$ to zero, as this 
could affect the output of the $\min$ and $\max$ networks!

Thus, to use this framework to obtain an architecture that is capable of 
implementing $\mu_{\text{MPC}}$, one needs to simply upper-bound the number of 
local linear functions $N$ in $\mu_{\text{MPC}}$ (ultimately without solving 
the actual MPC problem) and upper-bound the number of unique order regions $M$ 
in $\mu_{\text{MPC}}$. This is precisely the AReN algorithm, as specified in 
Algorithm \ref{alg:main_algorithm}. The constituent functions 
\texttt{EstimateRegionCount} and \texttt{EstimateUniqueOrder} are described in 
detail in the subsequent sections Section \ref{sec:mpc_regions} and Section 
\ref{sec:uo_approx}, respectively. The implementation of the function 
\texttt{InferArchitecture} follows directly from the (constructive) discussion 
in this section.
\begin{figure*}
	\centering 
	\includegraphics[width=0.95\textwidth]{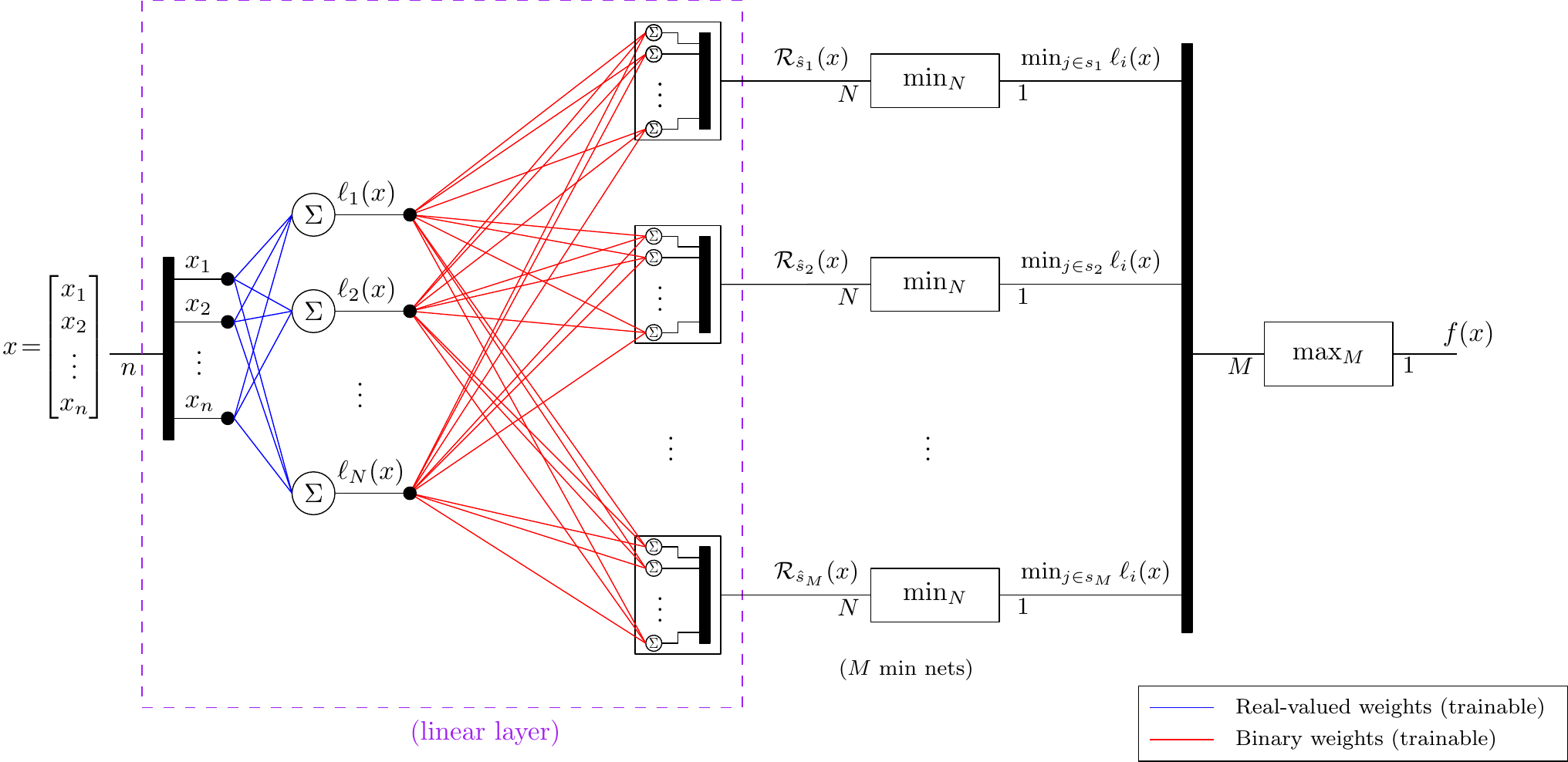} %
	\caption{Illustration of the overall architecture to implement a scalar 
	CPWL function. The symbols for various signals are indicated above the 
	line, and their dimensions are indicated below the line. The red lines 
	represent a fully-connected linear layer in which the weights \emph{flowing 
	into a single summer} have the property that \emph{exactly} one of them is 
	equal to $1$, and the others are all $0$. $\mathcal{R}_{\hat{s}_i}(x)$ is 
	defined in \eqref{eq:redundant_replication}.}
	\label{fig:overall_architecture} \vspace{-5mm}
\end{figure*}

\begin{algorithm}
\SetKwData{RegCnt}{M\_est}
\SetKwData{LinFnCnt}{N\_est}
\SetKwData{Arch}{ArchList}
\SetKwFunction{GetArch}{GetArchitecture}
\SetKwFunction{CntRegions}{EstimateRegionCount}
\SetKwFunction{CntUORegions}{EstimateUniqueOrderCount}
\SetKwFunction{InfArch}{InferArchitecture}
\SetKwInOut{Input}{input}
\SetKwInOut{Output}{output}
\Input{system matrices $A$,$B$,$C$; cost matrices $P,Q \geq 0$, $R>0$; feedback matrix $K$; horizon $N_c$}
\Output{$(K,\dim(\theta^{(1)}_1), \dots , \dim(\theta^{(K)}_1))$}
\BlankLine
\SetKwProg{Fn}{function}{}{end}%
\Fn{\GetArch{A,B,C,P,Q,R,K,$\text{N}_\text{c}$}}{
	\LinFnCnt $\leftarrow$ \CntRegions{A,B,C,P,Q,R,K,$\text{N}_\text{c}$}

	\RegCnt $\leftarrow$ \CntUORegions{\LinFnCnt}

	\Arch $\leftarrow$ \InfArch{\LinFnCnt,\RegCnt}

	\Return \Arch
}
\caption{AReN.}
\label{alg:main_algorithm}
\end{algorithm}




\section{Approximating the number of linear regions in the MPC controller} 
\label{sec:mpc_regions}
In this section, we will discuss our implementation of the function 
\texttt{EstimateRegionCount} from Algorithm \ref{alg:main_algorithm}. A natural 
means to approximate to the number of \emph{local linear functions} of 
$\mu_\text{MPC}$ is to approximate the number of \emph{maximal linear regions} 
in $\mu_\text{MPC}$. 

\begin{definition}[Linear Region of $\mu_\text{MPC}$]
	\label{def:linear_region}
	A \textbf{linear region} of $\mu_\text{MPC}$ is a subset of $\mathscr{R} 
	\subseteq \mathbb{R}^n$ over which $\mu_\text{MPC}(x) = L(x)$ for some 
	linear (affine) $L : \mathbb{R}^n \rightarrow \mathbb{R}^m$. A 
	\textbf{maximal linear region} is a linear region that is strictly 
	contained in no other linear regions. Two linear regions are said to be 
	\textbf{distinct} if they correspond to different linear functions, $L$.
\end{definition}
\noindent Thus, the maximal linear regions of $\mu_\text{MPC}$ are in 
one-to-one correspondence with the \emph{local linear functions} in 
$\mu_\text{MPC}$ (Definition \ref{def:local_linear_function}), so an upper 
bound on the number of maximal linear regions in $\mu_\text{MPC}$ is an upper 
bound on its number of local linear function, which in turn will provide an 
over-approximation of $N$ that can be used to generate a NN architecture.

To upper bound the number of maximal linear regions effectively, we need to consider in detail \emph{some} specifics 
about how the piecewise linear property arises in the solution for 
$\mu_\text{MPC}$. Ultimately, $\mu_\text{MPC}$ is piecewise linear because we 
have posed a problem for which (i) the gradient of the Lagrangian 
\eqref{eq:qp_grad} is linear in \emph{both} the Lagrange multipliers \emph{and} 
the decision variable; and (ii) the dependence on the initial state $x(t)$ is 
linear. Linearity is important in both (i) and (ii) because we are really not 
solving one optimization problem but a \emph{family} of them: one for each 
initial state $x(t)$. Thus, the linearity of the Lagrangian together with the 
linearity of the inequality constraints in $x(t)$ leads to an equation (a 
necessary optimality condition) that is linear in both the Lagrange multipliers 
and the initial state $x(t)$: hence the piecewise-linear controller 
$\mu_\text{MPC}$.

Moreover, the distinct linear regions of $\mu_\text{MPC}$ -- i.e. those with 
distinct linear functions -- arise out of a particular aspect of the 
aforementioned linear equations. In particular, the Lagrange multipliers, 
$\lambda$, and the initial state, $x(t)$, appear together in a linear equation 
that has different solutions -- and hence creates different linear regions for 
$\mu_\text{MPC}$ -- based on which of the inequality constraints are 
\emph{active} (at a particular optimizer) \cite[Theorem 
2]{BemporadExplicitLinearQuadratic2002}. Since the linear regions obtained in 
this way \emph{partition} the domain of $\mu_\text{MPC}$ (see also Proposition 
\ref{prop:linear_region_partition} below), this suggests that we can 
over-approximate the number of linear regions in $\mu_\text{MPC}$ by counting 
all of the possible constraints that can be active at the same time. Indeed, 
this is more or less how \texttt{EstimateRegionCount} arrives at an estimate 
for $N$, although we do not simply over-approximate with $2^\text{\# of 
constraints}$.

\subsection{The Optimal MPC Controller} 
\label{sub:the_optimal_mpc_controller}
As preparation for the rest of the section, we begin by summarizing some 
further details regarding the solution of $\mu_\text{MPC}$ from 
\cite{BemporadExplicitLinearQuadratic2002}. In particular, the optimization 
problem specified by \eqref{eq:cost_fn}-\eqref{eq:last_constraint} can be 
simplified by directly substituting the dynamics constraint 
\eqref{eq:dynamics_constraint} to get:
\begin{align}
 	&\min_U \left\{\frac{1}{2} x(t)^\text{T} \; Y \; x(t) + \frac{1}{2} U^\text{T} \; H \; U + x(t)^\text{T} \; F \; U \right\} \label{eq:simplified_qp}\\
 	&\text{subject to: } \qquad
 	\quad G U \leq W + E x(t) \notag
\end{align}
with appropriately defined matrices $H$, $F$, $G$, $W$ and $E$ of dimensions 
$(\rho \times \rho)$, $(n \times \omega)$, $(\rho \times \omega)$, $(\rho 
\times 1)$ and $(\rho \times n)$, respectively (where $\omega$ and $\rho$ are defined in~\eqref{eq:omega} and \eqref{eq:rho}, respectively). Then, completing the square by 
means of the change of variables $z \triangleq U + H^{-1} F^\text{T} x(t)$ 
provides the following, simplified quadratic program 
\cite{BemporadExplicitLinearQuadratic2002}:
\begin{equation}
\label{eq:main_qp}
	\min_z \left\{\frac{1}{2} z^\text{T} \; H \; z \right\}
	\quad
 	\text{ subject to: } 
 	\quad G z \leq W + S x(t).
\end{equation}
Note that this change of variables is only valid when $H$ is invertible, but 
this is assumed in \cite{BemporadExplicitLinearQuadratic2002}, and it will be 
assumed throughout.

A solution to the optimization problem in \eqref{eq:main_qp} can be easily 
formulated using the KKT (necessary) optimality conditions 
\cite{LuenbergerOptimizationVectorSpace1997}. This results in the following 
system of equations:
\begin{align}
	Hz + G^\text{T}\lambda &= 0 \qquad \lambda \in \mathbb{R}^\rho \label{eq:qp_grad}\\
	\lambda_i (G_i z - W_i - S_i x) &= 0 \qquad i \in 1 \ldots \rho \label{eq:multiplier} \\
	\lambda &\geq 0 \\
	G z &\leq W + S x . \label{eq:orig_inequality_constraints}
\end{align}
That is for a (local) minimizer to the optimization problem \eqref{eq:main_qp}, 
there exists non-negative Lagrange multipliers $\lambda$ that solve the above 
system. Moreover, solving \eqref{eq:qp_grad} for the minimizer, $z$, 
demonstrates that such a minimizer has a particular structure. Indeed, under 
the assumption that $H$ is invertible, we may solve \eqref{eq:qp_grad} for $z$ 
and substituted it into \eqref{eq:orig_inequality_constraints} to obtain:
\begin{equation}
\label{eq:inequality_z_free}
	-G H^{-1} G^\text{T} \lambda \leq W + S x.
\end{equation}

Now given the relevance of \emph{active constraints} in \eqref{eq:qp_grad} - 
\eqref{eq:orig_inequality_constraints} to\cite[Theorem 
2]{BemporadExplicitLinearQuadratic2002}, we introduce the following notation.
\begin{definition}[Selector Matrix]
	Let $\alpha \subseteq \{1, \dots, \rho\}$, and define the auxiliary set 
	(and associated vector):
	\begin{align}
		\tilde{\alpha} &\triangleq \{ (1, \min_{j \in \alpha}j), \dots, (|\alpha|, \max_{j \in \alpha} j) \}. \notag
	\end{align}
	Then the \textbf{selector matrix}, $\tilde{I}_\alpha$ is defined as:
	\begin{equation}
		\tilde{I}_\alpha \triangleq [e_{\tilde{\alpha}(1)}, \dots, e_{\tilde{\alpha}(|\alpha|)}]^\text{T}
	\end{equation}
	where $e_j$ is the $j^\text{th}$ column of the $(\rho \times \rho)$ 
	identity matrix.
\end{definition}

The selector matrix can thus be used to describe equality in 
\eqref{eq:inequality_z_free} for the active constraints by removing 
inequalities associated with the inactive constraints. In particular, given a 
set of active constraints specified by a subset $\alpha \subseteq \{1, \dots, 
\rho \}$, the Lagrange multipliers for the active constraints, 
$\lambda_\alpha$, will satisfy the following:
\begin{equation}
	\label{eq:first_approx}
	- \tilde{I}_\alpha G (H^{-1} G^\text{T} \tilde{I}_\alpha^\text{T}) \lambda_\alpha = \tilde{I}_\alpha (W + S x).
\end{equation}
In particular, \eqref{eq:first_approx} is a linear equation in 
$\lambda_\alpha$ and $x$ that ultimately specifies the region in $x$-space 
over which a single affine function characterizes $\mu_\text{MPC}$. Indeed, 
back substituting solutions of \eqref{eq:first_approx} into 
\eqref{eq:orig_inequality_constraints} and $\lambda \geq 0$ specifies a set of 
linear inequalities in $x$. These equivalently define an intersection of 
half-spaces in $\mathbb{R}^n$ that characterize a (convex) linear region of 
$\mu_\text{MPC}$ \cite[Theorem 2]{BemporadExplicitLinearQuadratic2002}. 
Moreover, since every possible solution to the optimization problem admits some 
set of active constraints, these convex linear regions partition the state 
space.

For our purposes, \eqref{eq:first_approx} is the relevant consequence of this 
discussion, since it most readily suggests how to over-approximate the number 
of linear regions of $\mu_\text{MPC}$. We describe how to do this in the next 
subsection.


\subsection{Over-approximating the Number of Maximal Linear Regions} 
\label{sub:over_approximating_the_sets_of_active_constraints}
From the previous discussion, the problem of finding all possible sets of 
active constraints for \eqref{eq:qp_grad} - 
\eqref{eq:orig_inequality_constraints} is a significant amount of the work in 
solving for the optimal controller $\mu_\text{MPC}$. However, we are content 
not solving for $\mu_\text{MPC}$ \emph{exactly}: thus, we only need to simplify 
\eqref{eq:first_approx} in such a way that we obtain a new equation with all of 
the same solutions plus some spurious ones (keep in mind that 
\eqref{eq:first_approx} is an equation in $\alpha$, too).

Before we begin counting regions, we need to state the following 
proposition, which is trivial given the observations in 
\cite{BemporadExplicitLinearQuadratic2002}.
\begin{proposition}
\label{prop:linear_region_partition}
	Let $\bar{\mathscr{R}}$ be a \emph{maximal} linear region for controller 
	$\mu_\text{MPC}$. Then there exists a finite collection of sets 
	$\Gamma_{\bar{\mathscr{R}}} = \{ \alpha_i \subseteq \{1, \dots, \rho \}: i 
	= 1, \dots V \}$ with the following property:

	\begin{itemize}
		\item for every $x \in \bar{\mathscr{R}}$, there exists an $\alpha_x 
			\in \Gamma$ and $|\alpha_x|$ Lagrange multipliers 
			$\lambda_{\alpha_x} \geq 0$ such that:
				\begin{equation}
					\tilde{I}_{\alpha_x} G H^{-1} G^\text{T} \tilde{I}_{\alpha_x}^\text{T} \lambda_{\alpha_x} = \tilde{I}_{\alpha_x} (W + S x).
					\label{eq:40}
				\end{equation}
	\end{itemize}
	In particular, any maximal linear region of $\mu_\text{MPC}$ can be 
	partitioned into $|\Gamma|$ convex linear regions.
\end{proposition}

\begin{proof}
	Since we are considering a maximal linear region of $\mu_\text{MPC}$, the 
	quadratic program \eqref{eq:main_qp} is feasible for every $x \in  
	\bar{\mathscr{R}}$ by definition. Consequently, there is for every such 
	$x$, a  unique solution $z^*_x$\footnote{This is because $H$ is positive 
	definite \cite[pp. 9]{BemporadExplicitLinearQuadratic2002}.}, and so by 
	necessity, there is some set of constraints $\alpha_x \subseteq \{1, \dots, 
	\rho\}$ that is active at $z^*_x$. Moreover, by the KKT necessary 
	conditions, there exists $|\alpha|$ Lagrange multipliers 
	$\lambda_{\alpha_x}$ that satisfy \eqref{eq:40}. This proves the 
	existential assertions for $x \in \bar{\mathscr{R}}$ related to 
	\eqref{eq:40}.

	However, we have implicitly defined a \emph{function} that associates to 
	each $x \in \bar{\mathscr{R}}$ a subset in $2^{\{1, \dots, \rho\}}$:
	\begin{align}
		\textnormal{act}: \mathbb{R}^n &\rightarrow 2^{\{1, \dots, \rho\}} \notag \\
		x &\mapsto \alpha_x.
	\end{align}
	We define $\Gamma_{\bar{\mathscr{R}}} = 
	\textnormal{act}(\bar{\mathscr{R}})$ to be the range of $\textnormal{act}$, 
	so that equivalence modulo $\textnormal{act}$  partitions 
	$\bar{\mathscr{R}}$ into $|\Gamma|$ disjoint regions. Finally, by 
	\cite[Theorem 2]{BemporadExplicitLinearQuadratic2002} and the discussion 
	about degeneracy on \cite[pp. 9]{BemporadExplicitLinearQuadratic2002}, each 
	of these regions is necessarily convex.
\end{proof}

Proposition \ref{prop:linear_region_partition} gives us a hint about how to 
over-approximate the number of maximal linear regions: in particular, we will 
simplify equation \eqref{eq:40} in such a way that we still find solutions for 
each $\alpha \in \Gamma$ at the expense of including solutions for $\alpha 
\not\in \Gamma$. This gives us our first main counting theorem:

\begin{theorem}
\label{thm:first_over-approx_theorem}
	Let $\Xi \subseteq 2^{\{1, \dots, \rho\}} \backslash \emptyset$ such that 
	for every $\alpha \in \Xi$ there exists a vector $\eta_\alpha \in \mathbb{R}^\rho$ such that:
	\begin{equation}
	\label{eq:second_overapprox}
		G H^{-1} G^\text{T} \tilde{I}_\alpha \lambda_\alpha = \eta_\alpha.
	\end{equation}
	has a solution $\lambda_\alpha \geq 0$, $\lambda_\alpha \neq 0$. Then 
	$|\Xi| + 1$ upper bounds the number of maximal linear regions for 
	$\mu_\text{MPC}$.
\end{theorem}
\begin{proof}
	If we show that for every maximal linear region $\bar{\mathscr{R}}$, the 
	$\Gamma_{\bar{\mathscr{R}}} \subseteq \Xi \cup \emptyset$, then the 
	conclusion will follow.

	But this follows directly from Proposition 
	\ref{prop:linear_region_partition}: for each $\alpha \in 
	\Gamma_{\bar{\mathscr{R}}}$, there exists an $x$ and $\alpha =  
	\lambda_{\alpha_x}$ such that \eqref{eq:40} holds. Thus if $\lambda_\alpha 
	= \lambda_{\alpha_x}$ and $\lambda_{\alpha_x} \neq 0$, then set $\alpha = 
	\alpha_x$, $\eta_\alpha = W + S x$ and conclusion holds. The situation when 
	no constraints are active is accounted for with the addition of $1$ in the 
	final conclusion.
\end{proof}

Theorem \ref{thm:first_over-approx_theorem} is significant because Farkas' 
lemma \cite{LuenbergerOptimizationVectorSpace1997} tells us how to describe the 
solutions of \eqref{eq:second_overapprox} using a linear inequality feasibility 
problem. In particular, \eqref{eq:second_overapprox} has a non-trivial solution 
if and only if the problem $(G H^{-1} G^\text{T} 
\tilde{I}_\alpha^\text{T})^\text{T}\chi \leq 0$ is feasible (and it can't have 
\emph{only} the trivial solution for non-trivial $\lambda_\alpha$). This 
reasoning is included in the proof of the subsequent Theorem, which connects 
the bound from Theorem \ref{thm:first_over-approx_theorem} to the number of 
feasible ``sub-problems'' defined by $G H^{-1} G^\text{T}$. First, we introduce 
the following definition.

\begin{definition}[Non-trivially Feasible]
\label{def:non-trivially_feasible}
	Let $V$ be a $(\rho \times \rho^\prime)$ matrix and let $\alpha \subseteq 
	\{1, \dots, \rho\}$. Then $\alpha$ is a \textbf{non-trivially feasible 
	subset of} $V$ if there exists a $\chi \in \mathbb{R}^{\rho^\prime}$, $\chi 
	\neq 0$ such that $I_\alpha V \chi \geq 0$. Such a set is \textbf{maximal} 
	if adding any other row makes it infeasible.
\end{definition}

Now we state the main theorem in this section.

\begin{theorem}
	\label{thm:second_over-approximation_theorem}
	Let $\mathscr{I}$ be the set of maximal non-trivially feasible subsets of 
	$G H^{-1} G^\text{T}$ (see Definition \ref{def:non-trivially_feasible}). 
	Then the number of maximal linear regions in $\mu_\text{MPC}$ is bounded 
	above by:
	\begin{equation}
		\Big|\bigcup_{\alpha \in \mathscr{I}} 2^\alpha \Big| \leq 
		\sum_{\alpha \in \mathscr{I}} 2^{|\alpha|}.
	\end{equation}
\end{theorem}
\begin{proof}
	This follows from Farkas' lemma and Theorem 
	\ref{thm:first_over-approx_theorem}.

	In particular, let $\alpha \subseteq \{1, \dots, \rho\}$ and $\eta_\alpha 
	\neq 0 \in \mathbb{R}^{\rho}$. Now, by Farkas' lemma,
	\begin{multline}
		\exists \lambda_\alpha \geq 0 \;.\; G H^{-1} G^\text{T} \tilde{I}_\alpha \lambda_\alpha^\text{T} = \eta_\alpha \Leftrightarrow \\
				\exists \chi \in \mathbb{R}^\rho \text{ s.t. } \eta_\alpha^\text{T} \chi < 0 \text{ and } \tilde{I}_\alpha (G H^{-1} G^\text{T})^\text{T} \chi \leq 0.
	\end{multline}
	In particular, $G H^{-1} G^\text{T} \tilde{I}_\alpha 
	\lambda_\alpha^\text{T} = \eta_\alpha$ can have a non-negative solution if 
	and only if $\alpha$ is a \textbf{non-trivially feasible} subset of 
	$(GH^{-1}G^\text{T})^\text{T}$.

	Thus, by Theorem \ref{thm:first_over-approx_theorem}, we conclude that 
	$\alpha \in \Xi$ implies $\alpha$ is a non-trivially feasible subset of $(G 
	H^{-1} G^\text{T})^\text{T}$, and hence, \emph{that the number of maximal 
	linear regions is bounded by the number of non-trivially feasible subsets 
	of} $G H^{-1}G^\text{T}$. The conclusion of the theorem thus follows 
	because every non-trivially feasible subset is a subset of some maximally 
	non-trivially feasible subset.
\end{proof}

\subsection{Implementing \texttt{EstimateRegionCount}} 
\label{sub:implementing_text}

To find maximal non-trivially feasible subsets of $GH^{-1}G^T$, we start by introducing one Boolean variable $b_i$ for each column of the matrix $GH^{-1}G^T$. The first non-trivially feasible subset can then be found by solving the following problem:
\begin{align}
	&\text{arg}\max_{(b_1,\ldots,b_\rho,\lambda)\in \mathbb{B}^\rho \times \mathbb{R}^{\omega}} \sum_{i = 1}^\rho b_i \label{eq:obj_B}\\
	&\text{subject to} \qquad
	b_i \Rightarrow [GH^{-1}G^T]_i \lambda < 0, \;\; i = 1,\ldots,\rho \label{eq:column_constraint}
\end{align}
where $[GH^{-1}G^T]_i$ denotes the $i$th column of the matrix $GH^{-1}G^T$. Such optimization problems can be solved efficiently using Satisfiability Modulo Convex programming (SMC) solvers~\cite{shoukry2018smc,shoukry2017smc}. SMC solvers first use a pseudo-Boolean Satisfiability (SAT) solver to find a valuation of the Boolean variables that maximizes the objective function~\eqref{eq:obj_B}; this is then followed by a linear programming (LP) solver that finds solutions to the constraints in~\eqref{eq:column_constraint}. Indeed, the SAT solver may return an assignment for the Boolean variables $b_1,\ldots,b_{\rho}$ for which the corresponding LP problem is infeasible. In such a case, we use the LP solver to search for a set of Irreducibly Infeasible Set (IIS) that explains the reason behind such infeasibility. This IIS is then encoded into a constraint that prevents the SAT solver from returning any assignment that can lead to the same IIS. We iterate between the SAT solver and the LP solver until one Boolean assignment is found for which the corresponding LP problem is feasible. It follows from maximizing the objective function that this set of active constraints is guaranteed to be a maximal non-trivially feasible subset of $GH^{-1}G^T$.

Once a maximal non-trivially feasible subset of $GH^{-1}G^T$ is found, we can add a blocking Boolean constraints to the SAT solver, thus preventing the SAT solver from producing any subsets of this maximal non-trivially feasible set. We continue this process until the SAT solver can not find any more feasible assignments to the Boolean variables, at which point our algorithm terminates and returns all of the non-trivially feasible subsets it has obtained. This discussion is summarized in Algorithm~\ref{alg:estimateRegionCount} whose correctness follows from the correctness of SMC solvers~\cite{shoukry2018smc,shoukry2017smc}.

\begin{algorithm}[!t]

\SetKwData{LinFnCnt}{N\_est}
\SetKwData{hyperplanes}{G\_Hinv\_Gtr}
\SetKwData{solnsFound}{AllSolutionsFound}
\SetKwData{numhyperplanes}{NumHyperplanes}
\SetKwData{h}{h}
\SetKwData{b}{b}
\SetKwData{false}{False}
\SetKwData{true}{True}
\SetKwData{sols}{Solutions}
\SetKwData{cons}{SATConstraints}
\SetKwData{setcons}{setConstraints}
\SetKwData{feas}{Feasible?}
\SetKwData{iis}{IIS}
\SetKwData{z}{z}
\SetKwData{hyperset}{HyperplaneSet}

\SetKwFunction{CntRegions}{EstimateRegionCount}
\SetKwFunction{GetHyperplanes}{GetHyperplanes}
\SetKwFunction{dim}{Dimensions}
\SetKwFunction{sat}{SATsolver}
\SetKwFunction{maxx}{Maximize}
\SetKwFunction{init}{init}
\SetKwFunction{initbools}{createBooleanVariables}
\SetKwFunction{append}{Append}
\SetKwFunction{satq}{SAT?}
\SetKwFunction{checkfeas}{CheckFeasibility}
\SetKwFunction{truevars}{TrueVars}
\SetKwFunction{selecthypers}{GetHyperplanes}
\SetKwFunction{getiis}{GetIIS}
\SetKwFunction{cntunique}{CountAllUniqueSubsets}

\SetKw{Break}{break}

\SetKwInOut{Input}{input}
\SetKwInOut{Output}{output}
\Input{system matrices $A$,$B$,$C$; cost matrices $P,Q \geq 0$, $R>0$; feedback matrix $K$; horizon $N_c$}
\Output{\LinFnCnt}
\BlankLine
\SetKwProg{Fn}{function}{}{end}%
\Fn{\CntRegions{A,B,C,P,Q,R,K,$\text{N}_\text{c}$}}{
	\hyperplanes $\leftarrow$ \GetHyperplanes{A,B,C,P,Q,R,K,$\text{N}_\text{c}$}

	\numhyperplanes $\leftarrow$ \dim{\hyperplanes}[0]

	\BlankLine
	$($ \h $\negthinspace[1], \dots ,$ \h $\negthinspace[\texttt{NumHyperplanes}] ~)$ $\leftarrow$ \hyperplanes

	$($ \b $\negthinspace[1], \dots ,$ \b $\negthinspace[\texttt{NumHyperplanes}] ~)$ $\leftarrow$ 	\initbools{\numhyperplanes}


	\sols $\leftarrow$ ( ) ; \cons $\leftarrow$ ( )


	\While{\true}{
		\sat.\setcons{\cons}

		\sat.\maxx{$\sum_{i=1}^{{\textnormal{\textsf{NumHyperplanes}}}}$ \b$\negthinspace[i]$}

		\If{not \sat.\satq{}}{
			\Break
		}


		\hyperset $\leftarrow$ \hyperplanes. 

		$\qquad\quad$\selecthypers{\sat.\truevars{}}

		\feas $\leftarrow$ \checkfeas{\hyperset\hspace{-2pt}*\z $\leq -\varepsilon$}

		\uIf{not \feas}{

			\iis $\leftarrow$ \getiis{\hyperset}

			\cons.\append{ $\bigvee_{\texttt{h}[i] \in {\textnormal{\textsf{IIS}}}} \neg {\texttt{b}}[i]$  }

		}\Else{
			\sols.\append{\hyperset}

			\cons.\append(

			$\quad\sum_{\texttt{h}[i] \in {\texttt{HyperplaneSet}}}\negthinspace{\texttt{b}}[i]\negthinspace<\negthinspace|\textnormal{\textsf{HyperplaneSet}}|\negthinspace\Rightarrow$

			$\qquad\quad \sum_{\texttt{h}[i] \not\in {\textnormal{\textsf{HyperplaneSet}}}}{\texttt{b}}[i]\geq 1$

			)
		}

	}

	\Return \LinFnCnt $\leftarrow$ \cntunique{\sols}
}
\caption{\texttt{EstimateRegionCount}.}
\label{alg:estimateRegionCount}
\end{algorithm}




\section{Approximating the Number of Unique-Order Regions in the MPC Controller} 
\label{sec:uo_approx}
In this section, we discuss our implementation of \texttt{EstimateUniqueOrder} 
from Algorithm \ref{alg:main_algorithm}. Unlike our implementation of 
\texttt{EstimateRegionCount}, which we could base of aspects of the MPC 
problem, the implementation of this function merely exploits a general bound on 
the number of possible regions in an arrangement of hyperplanes.

In particular, we noted in Definition \ref{def:order_region} that the 
unique-order regions created by a set of local linear functions $\mathcal{R} = 
\{\ell_1, \dots, \ell_N\}$ correspond to the regions in the hyperplane 
arrangement specified by \emph{non-empty} hyperplanes of the form $H_{ij} = \{ 
x : \ell_i(x) = \ell_j(x) \}$, each of which is a hyperplane in dimension of 
$n$ (when $x \in \mathbb{R}^n$).

There seems to be a well-known -- but rarely stated -- upper bound on the 
number of regions that can be formed by a hyperplane arrangement of $N$ 
hyperplanes in dimension $n$. The few places where it is stated (e.g 
\cite[Lemma 4]{SerraBoundingCountingLinear2018}) seem to ambiguously quote 
Zaslavsky's Theorem \cite[Thoerem 
2.5]{StanleyIntroductionHyperplaneArrangements} in their proofs. Thus, we state 
the bound, and sketch a proof.

\begin{theorem}
	\label{thm:main_uo_bound}
	Let $\mathcal{A}$ be an arrangement of $N$ hyperplanes in dimension $n$. 
	Then the number of regions created by this arrangement, $r(\mathcal{A})$ is 
	bounded by:
	\begin{equation}
	\label{eq:uo_bound}
		r(\mathcal{A}) \leq \sum_{i = 0}^{n} {N \choose i}
	\end{equation}
	(with equality if and only if $\mathcal{A}$ is in general position 
	\cite[pp. 4]{StanleyIntroductionHyperplaneArrangements}).
\end{theorem}
\begin{proof}
	First, we note that the bound holds with equality for arrangements in 
	\emph{general position} (defined on \cite[pp. 
	4]{StanleyIntroductionHyperplaneArrangements}); this is from 
	\cite[Proposition 2.4]{StanleyIntroductionHyperplaneArrangements}, a 
	consequence of Zaslavsky's theorem \cite[Theorem 
	2.5]{StanleyIntroductionHyperplaneArrangements}. Thus, the claim holds if 
	every other arrangement has fewer regions than an arrangement in general 
	position with the same number of hyperplanes.

	This is indeed the case, but it helps to have a little bit of terminology 
	first. In particular, we introduce the general formula for the number of 
	regions in a hyperplane arrangement, $r(\mathcal{A})$, in terms of a triple 
	of hyperplane arrangements $(\mathcal{A}, \mathcal{A}^\prime, 
	\mathcal{A}^{\prime\prime})$ \cite[pp. 
	13]{StanleyIntroductionHyperplaneArrangements}, namely \cite[Lemma 
	2.1]{StanleyIntroductionHyperplaneArrangements}:
	\begin{equation}
	\label{eq:hyper_triple_rel}
		r(\mathcal{A}) = r(\mathcal{A}^\prime) + r(\mathcal{A}^{\prime\prime}).
	\end{equation}
	Such a triple is formed by choosing a distinguished hyperplane $H_d \in 
	\mathcal{A}$, and defining $\mathcal{A}^\prime$ as 
	$\mathcal{A}\backslash\{H_d\}$ and $\mathcal{A}^{\prime\prime}$ as the 
	arrangement of hyperplanes $\{H \cap H_d \neq \emptyset : H \in 
	\mathcal{A}^\prime\}$. Note that $\mathcal{A}^{\prime\prime}$ characterizes 
	the regions in $\mathcal{A}^\prime$ that are \emph{split} by $H_d$.

	From here, we will only provide a brief proof sketch. The proof proceeds by 
	induction: first on the number of hyperplanes in $n=2$, and then on by 
	induction on the dimension, $n$. For $n=2$, the result can be shown for 
	arrangements of size $N$ using \eqref{eq:hyper_triple_rel}, and noting that 
	$r(\mathcal{A}^{\prime\prime}) = N$ if and only if $H_d$ intersects all the 
	other hyperplanes exactly once. This, together with the induction 
	assumption, shows $r(\mathcal{A})$ can satisfy the claim with equality only 
	if $\mathcal{A}$ is in general position. For $n > 2$, the proof proceeds 
	similarly, using \eqref{eq:hyper_triple_rel} to invoke the conclusion for 
	$n-1$ as necessary.
\end{proof}

Thus, our implementation of \texttt{EstimateUniqueOrderCount} simply computes 
and returns the value in \eqref{eq:uo_bound}. In the worst case, this estimate 
is $2^\text{\# of hyperplanes}$: this occurs for example when $N=n$. But for $N 
>> n$, this bound clearly grows more slowly than exponentially in $N$. This is 
extremely helpful in keeping the size of the second linear layer in Figure 
\ref{fig:overall_architecture} of a reasonable size.

We conclude this section by noting that the result in Theorem 
\ref{thm:main_uo_bound} may be used to state Theorem \ref{thm:embedding} 
independently of $M$ entirely. In particular, we have the following theorem.
\begin{theorem}
	\label{thm:uo_free_overapprox}
	Let $f$ be a CPWL function, and let $\overbar{N}$ be an upper-bound on the 
	number of local linear functions in $f$. Then for $M^\prime = \sum_{i = 
	0}^{n} {N \choose i}$ there exists a parameter list 
	$\Theta_{\overbar{N},M^\prime}$ with 
	$\textnormal{Arch}(\Theta_{\overbar{N},M^\prime})$ as in 
	\eqref{eq:nn_representation} such that:
	\begin{equation}
		f(x) = \nn_{\Theta_{\overbar{N},M^\prime}}(x) \quad \forall x.
	\end{equation}
\end{theorem}



\section{Discussion: Hinging Hyperplane Implementations} 
\label{sec:discussion}

For comparison, we will make some remarks about the hinging hyperplane 
representation used in \cite[Theorem 2.1]{AroraUnderstandingDeepNeural2016}. 
\begin{theorem}[Hinging Hyperplane Representation \cite{ShuningWangGeneralizationHingingHyperplanes2005}]
\label{thm:hinging_hyperplane_representation}
	Let $f : \mathbb{R}^n \rightarrow \mathbb{R}$ be a CPWL function. Then 
	there exists a finite integer $K$ and $K$ non-negative integers $\{\eta_k : 
	k = 1, \ldots, K \}$, each less than $n+1$, such that
	\begin{equation}
	\label{eq:hyperplane_form}
		f(x) = \sum_{k = 1}^K \sigma_k \max \left\{ 
			\mathcal{L}^{(k)}_1(x) , \mathcal{L}^{(k)}_2(x), \ldots, \mathcal{L}^{(k)}_{\eta_k}(x)
		\right\} \quad \forall x \in \mathbb{R}^n
	\end{equation}
	for some collection of $\mathcal{L}^{(k)}_i$, $i \in \{1, \ldots, 
	\eta_k\}$, each of which is an affine function of its argument $x$, and 
	each constant $\sigma_k \in \{-1, +1\}$. (Each $\mathcal{L}^{(k)}_i$ beyond 
	the first one is referred to as a ``hinge''. That is $\eta_k - 1$ is the 
	number of ``hinges'' in the $k^\text{th}$ summand.)
\end{theorem}

The problem with creating a NN architecture from Theorem 
\ref{thm:hinging_hyperplane_representation} is that its proof provides only an 
\emph{existential} assertion -- by means of explicit construction -- and that  
construction relies heavily on particular knowledge of the \emph{actual} local 
linear functions in the CPWL \cite[Theorem 
1]{ShuningWangGeneralizationHingingHyperplanes2005}. Thus, the number $K$ in 
\eqref{eq:hyperplane_form} -- which is essentially the only 
\emph{architectural} parameter in the ReLU -- has a complicated dependence on 
the particular CPWL function (see \cite[Corollary 
3]{ShuningWangGeneralizationHingingHyperplanes2005} as it appears in the proof 
of \cite[Theorem 1]{ShuningWangGeneralizationHingingHyperplanes2005}). This 
even makes it difficult to ``replay'' the proof of Theorem 
\ref{thm:hinging_hyperplane_representation} and upper-bound the size of the 
existential assertions in each step as necessary; there are naive upper bounds 
for each step, but they lead to an extravagant number of $\max$ operations 
(exponentially many, in fact).

Moreover, a further complication is that hinging hyperplane representations 
need not even be unique. For example, consider $f(x) = |x|$: for this function, 
$K = 1$ works because $|x| = \max\{ x, -x \}$. But the same function could also 
be implemented as 
\begin{equation}
	|x| = \max\{-x, 3 x + 4 \} - \max\{0 x , 4 x + 4\} + \max\{ 0 x, 2 x\}
\end{equation}
which has $K=3$ and \emph{five} different linear functions $\mathcal{L}$. 
However, this clearly suggests an alternate approach to upper-bounding the 
steps in Theorem \ref{thm:hinging_hyperplane_representation}: create an 
architecture derived from the CPWL with the largest \emph{minimal} 
hinging-hyperplane representation. We conjecture that such a maxi-min hinging 
hyperplane form would in fact lead to \emph{fewer} $\max$ units than the 
lattice representation used in AReN. Unfortunately, as far as we are aware, 
there exists no such minimal characterization in the literature (as a function 
of the number of local linear functions, say), so we leave consideration of 
this problem to future work.



\section{Numerical Results} 
\label{sec:numerical_results}
The function \texttt{EstimateRegionCount} (Algorithm 
\ref{alg:estimateRegionCount}) is the bottleneck in Algorithm 
\ref{alg:main_algorithm}. Therefore, we chose to benchmark Algorithm 
\ref{alg:estimateRegionCount} to gage the overall performance of our proposed 
framework. We implemented \texttt{EstimateRegionCount} using SAT solver Z3 and 
convex solver CPLEX using their respective Python interfaces. We tested our 
implementation on single-input,single-output MPC problems in two contexts: (1) 
with a varying number of states; and (2) with a varying prediction horizon 
$N_c$. The computer used had an Intel Core i7 2.9-GHz processor and 16 GB of 
memory.

Figure \ref{fig:experiment_1} (top) shows the performance of Algorithm 
\ref{alg:estimateRegionCount} as a function of the number of plant states, $n$, 
with all other parameters held constant. The estimated number of local linear 
functions N\_est output by \texttt{EstimateRegionCount} is plotted on one axis; 
the maximum number of linear functions needed, $2^\rho$, is also shown for 
reference. The other axis shows the execution time for each problem in seconds. 
It follows from Theorem~\ref{thm:second_over-approximation_theorem} that the 
number of plant states doesn't change the number of constraints and hence does 
not contribute to the complexity of Algorithm \ref{alg:estimateRegionCount}. 
Note that Algorithm \ref{alg:estimateRegionCount} reported a number of local 
linear functions that is one order of magnitude less than the maximum number of 
linear functions needed, $2^\rho$ while taking less than 1.5 minutes of 
execution time.  

Figure \ref{fig:experiment_1} (bottom) shows the performance of our algorithm (in semi-log scale) as a function of the number of constraints, $\rho$, with all other parameters held constant ($n = 100$). The estimated number of linear functions output by \texttt{EstimateRegionCount} is plotted on one axis; the maximum number of  linear functions needed, $2^\rho$, is also shown for reference. The other axis shows the execution time for each problem in seconds. Again,  we notice an order of magnitude difference between the reported number of local linear functions versus the maximum number of linear functions needed, $2^\rho$. Indeed, the execution time is affected by increasing the number of constraint, nevertheless, Algorithm  \ref{alg:estimateRegionCount} terminates in less than 1.5 hours for a system with more than 300,000 maximal linear regions. 

\begin{figure}
	\includegraphics[width=0.99\columnwidth]{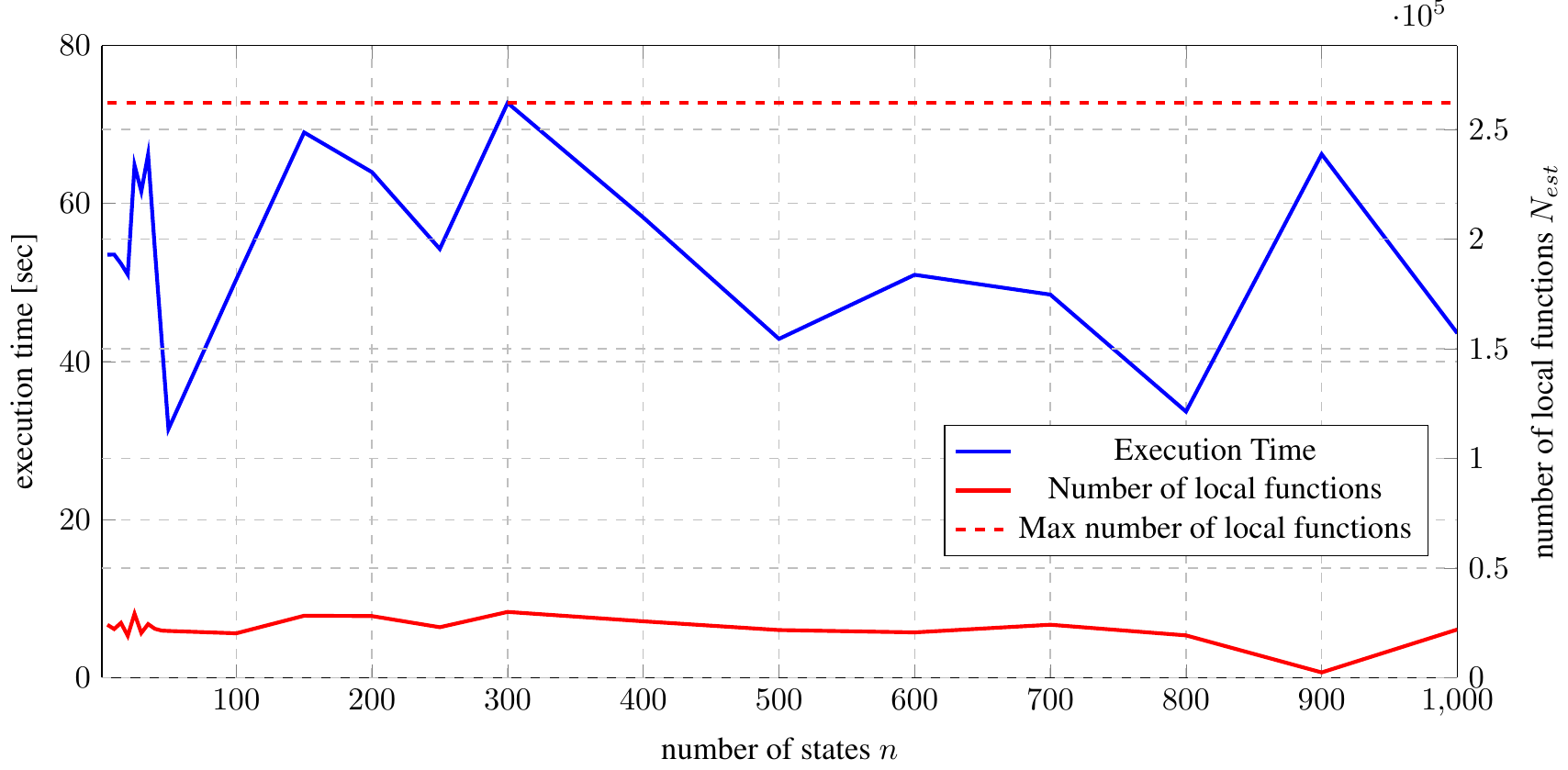} \\
	\includegraphics[width=0.99\columnwidth]{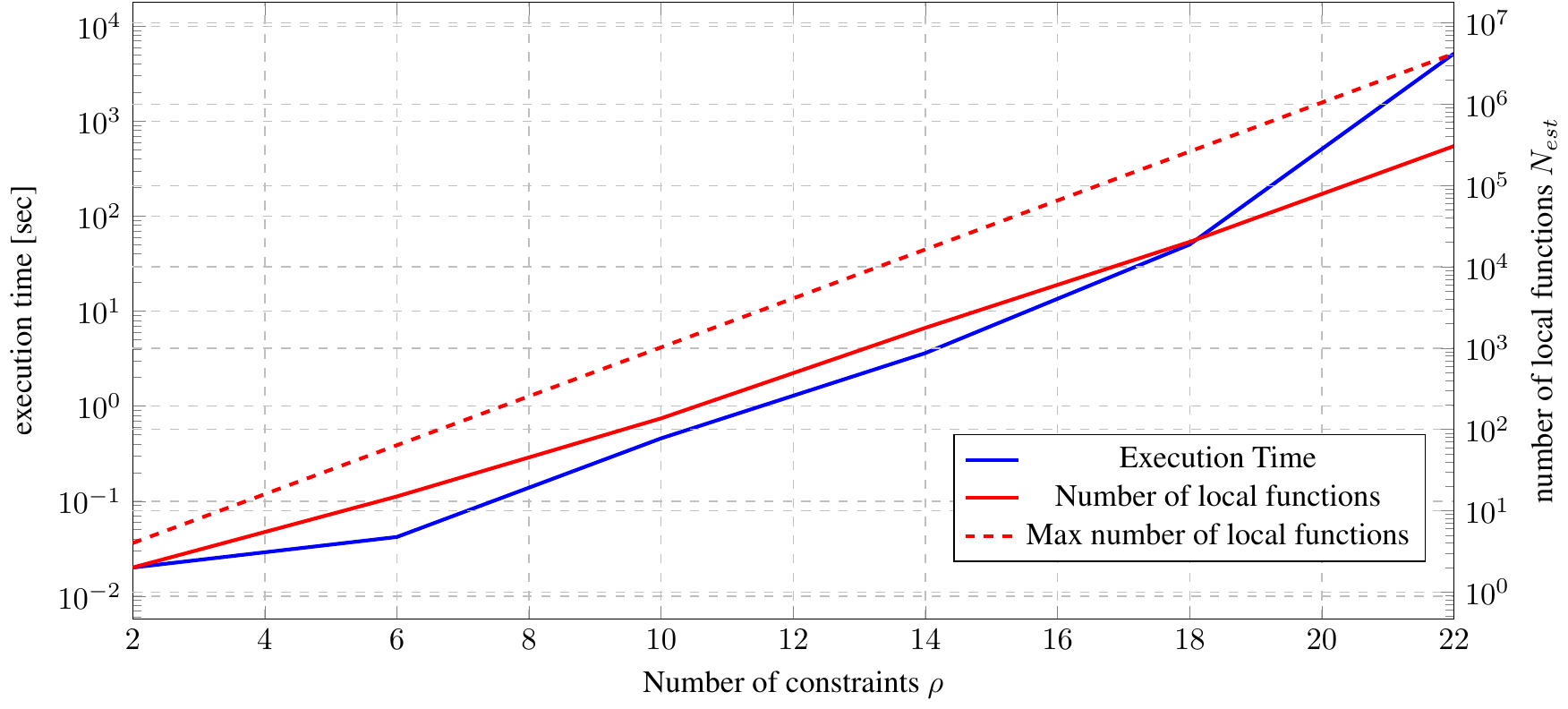}
	\caption{Execution time results: (top) when the number of states $n$ increases for a fixed number of constraints and (bottom) when the number of constraints $\rho$ increases for a fixed number of states $n = 100$.} \vspace{-5mm}
	\label{fig:experiment_1}
\end{figure}

%

\bibliographystyle{plain} %
\bibliography{mybib}

\end{document}